\documentclass{article}

\usepackage{lineno,hyperref}
\modulolinenumbers[5]

\usepackage{authblk}
\usepackage{amsmath, amsthm, amssymb}
\usepackage{bm}
\usepackage{graphicx}
\usepackage{xcolor}
\usepackage{xspace}
\usepackage{url}
\usepackage{multirow}

	\usepackage{algorithm}
	\usepackage{algorithmic}

\usepackage{hyperref}

\usepackage{amsmath, amsthm, amssymb}
\usepackage{bm}
\usepackage{graphicx}
\usepackage{xcolor}
\usepackage{xspace}
\usepackage{url}
\usepackage{multirow}

	\usepackage{algorithm}
	\usepackage{algorithmic}

\newcommand{\xvec}{\mathbf{x}}
\newcommand{\yvec}{\mathbf{y}}

\newcommand{\vvec}{\mathbf{v}}

\newtheorem{lemma}{Lemma}
\newtheorem{thm}{Theorem}

\newtheorem{prop}{Proposition}
\newtheorem{fact}{Fact}

\newtheorem{definition}{Definition}

\newcommand{\remove}[1]{}

\newcommand{\mom}[1]{{\left\vert\kern-0.25ex\left\vert\kern-0.25ex\left\vert #1 \right\vert\kern-0.25ex\right\vert\kern-0.25ex\right\vert}}

\begin{document}

\title{The computational complexity of some explainable clustering problems}


\date{}

\author{ Eduardo Laber \\PUC-Rio, Brazil\\ {\tt laber@inf.puc-rio.br} }

\maketitle


\begin{abstract}

We study the computational complexity of some explainable clustering problems 
in the framework proposed by [Dasgupta et al., ICML 2020], where  explainability is achieved via axis-aligned decision trees.  
We consider the $k$-means, $k$-medians, $k$-centers  and
the spacing cost functions. We prove that the first three are hard to optimize while the latter
can be optimized in polynomial time.
\end{abstract}




\section{Introduction}
Machine learning models and algorithms have been used in a number of systems that take decisions that affect our lives. Thus, explainable methods are desirable so that people are able to have a better understanding of their behavior, which allows for comfortable use of these systems or, eventually, the questioning of their applicability 
\cite{ribeiro2016should}.

Recently, there has been some effort to devise explainable methods for unsupervised learning tasks, in particular, for clustering \cite{dasgupta2020explainable,bertsimas2020interpretable}. We investigate the framework discussed  by \cite{dasgupta2020explainable}, where  an explainable clustering 
is given by a partition,  induced by the leaves of an axis-aligned decision tree, that optimizes
some predefined objective function.


Figure \ref{fig:exp-clustering}
shows a decision tree that defines  a clustering for the {\tt Iris} dataset. The clustering 
has three groups, each of them corresponding to a leaf. The explanation of the group associated with the rightmost leaf is
{\tt Sepal Length >0.4} AND {\tt Petal Width < 0.5}.

\begin{figure}
\begin{center}
	 \includegraphics[scale=0.5]{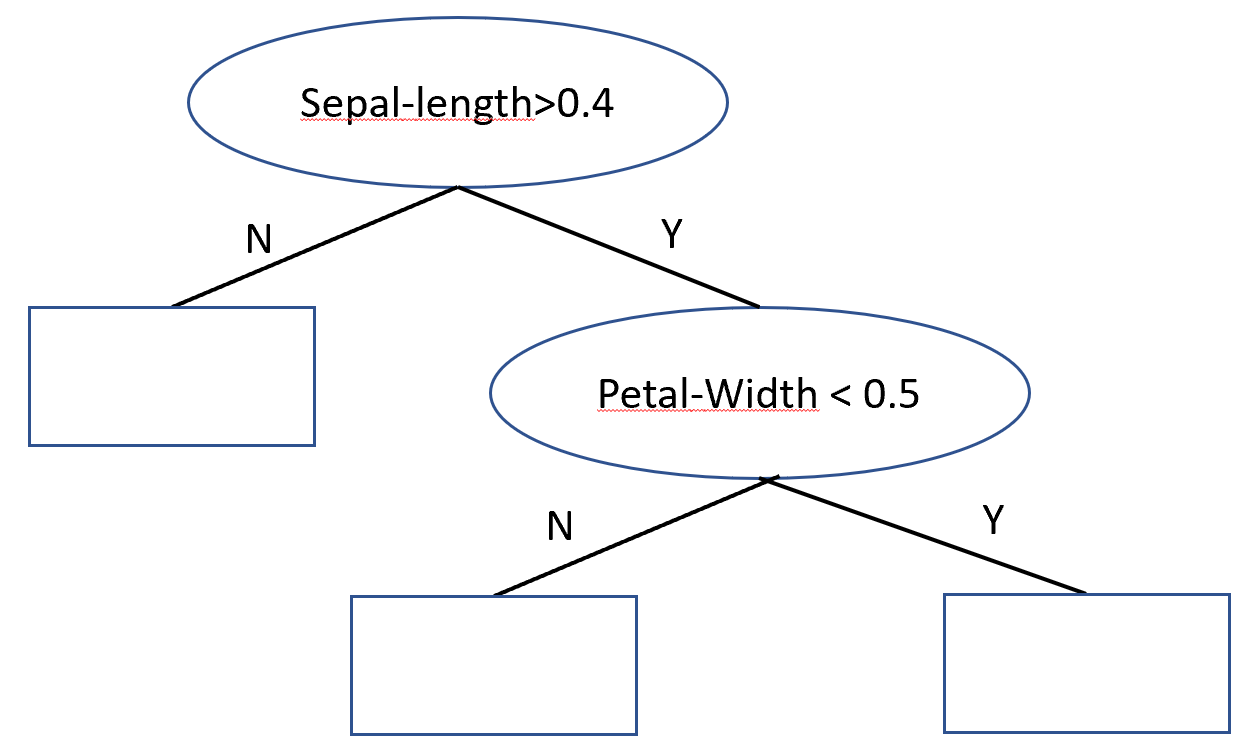}
\label{fig:exp-clustering}
\caption{An explainable clustering with 3 groups for the {\tt Iris} datasets}
	\end{center}

\end{figure}

\remove{ 
 
Figure \ref{fig:exp-clustering}
shows a clustering with three groups induced by a decision tree with $3$ leaves.
As an example, the blue cluster can be explained as the set of points  that satisfy
{\tt Feature 1} $\le 70$ and {\tt Feature 2} $> 40$.
Simple explanations as this one are usually not available
for the partitions produced by popular methods such as the Lloyd's algorithm for the $k$-means problem.

\begin{center}
		\hspace{40pt} \includegraphics[scale=0.5]{Exp-Clustering.png}
\label{fig:exp-clustering}
	\end{center}
 
 }

Following \cite{dasgupta2020explainable},
 a series of papers \cite{DBLP:conf/icml/LaberM21,DBLP:conf/icml/MakarychevS21,charikar2021near,esfandiari2021almost,gamlath2021nearly} provided algorithms, with provable
guarantees, to build
decision trees that induce explainable clustering.
Several cost functions  to guide the clustering construction  were investigated  as  the $k$-means, $k$-centers, $k$-medians and maximum-spacing.
Despite this active research,
the only work on the field that tackles    the computational complexity of building explainable clustering
is \cite{bandyapadhyay2021good}, where it was  proved that optimizing the  
$k$-means and the $k$-medians cost functions is NP-Complete. Here, we 
improve these results and also investigate the computational complexity
for both the $k$-centers and the spacing cost functions.

\remove{
\cite{dasgupta2020explainable} present a way of approximating solutions for the $k$-means and $k$-medians clustering problems using small decision trees -- i.e., allowing only cuts that are perpendicular to the planes, and limited by previous cuts, when defining the clusters' borders. They show that ``any tree-induced clustering must in general incur an $\Omega(\log k)$ approximation factor compared to the optimal clustering'', as well as ``an efficient algorithm that produces explainable clusters using $k$ leaves'' with a constant factor approximation for $k = 2$ and, for general $k \geq 2$, 
``an $O(k)$ approximation to the optimal k-medians and an
$O(k^2)$ approximation to the optimal k-means''. Here we expand on their work by analysing the approximation of $k$-centers clustering using decision trees.}

\subsection{Problem definition}
Let ${\cal X} $ be a finite set of  points in $ \mathbb{R}^d$.
We say that a decision tree is {\em standard} if each internal node $v$ is associated with a
test (cut), specified by a coordinate $i_v \in [d]$ and a real value $\theta_v$, that partitions the  points
in ${\cal X}$ that reach $v$ into two sets:
those  having the  coordinate  $i_v$  smaller than or equal to $\theta_v$ and
those having  it larger than $\theta_v$. The leaves of a standard
decision tree   induce a partition of  $ \mathbb{R}^d$ into axis-aligned boxes and,
naturally, a partition of ${\cal X} $ into  clusters.

Let  $k \ge 2$ be an integer.   The clustering problems considered here consist of finding a  partition
of ${\cal X}$ into $k$ groups, among those that can be induced by a standard decision tree with $k$ leaves,
 that optimizes a given objective function.
For the $k$-means, $k$-medians and $k$-centers cost functions, in addition to the partition, 
a representative $\mu(C) \in \mathbb{R}^d$ for each group $C$ must also
be output. 
 
For the $k$-means problem the objective (cost function) to be minimized is the Sum of the Squared Euclidean Distances (SSED) between each point $\xvec \in {\cal X}$ and the representative of the cluster
where $\xvec$ lies. Mathematically, the cost (SSED) of a partition 
${\cal C}=(C_1,\ldots,C_k)$ for ${\cal X}$ is given by
$$\sum_{i=1}^k \sum_{\xvec \in C_i} || \xvec-\mu(C_i) ||_2^2.$$




For the  $k$-medians problem the cost of a partition ${\cal C}=(C_1,\ldots,C_k)$ is given by 
$$\sum_{i=1}^k \sum_{\xvec \in C_i} || \xvec-\mu(C_i) ||_1.$$

The $k$-centers problem is  also a minimization problem;
its cost function for a partition ${\cal C}=(C_1,\ldots,C_k)$ is given by 
$$\max_{i=1,\ldots,k}  \max_{\xvec \in C_i} \{ || \xvec-\mu(C_i) ||_2 \} .$$




Let $dist:{\cal X} \times {\cal X} \mapsto \mathbb{R}^+$ be a distance function.
 The meximum-spacing problem consists of finding a partition with at least $k$ groups
 that has maximum spacing, where  the spacing
$sp({\cal C})$ of a partition  ${\cal C}$ is 
 defined as 
$$ sp({\cal C})= \min \{ dist(\xvec,\yvec) : \xvec \mbox{ and } \yvec \mbox{ lie in distinct groups of ${\cal C}$}
 \}$$
In contrast to the other criteria, the spacing is an inter-clustering criterion.

We note that an optimal solution of the unrestricted version
 of any of these problems, in which the decision tree constraint is not enforced, might be a partition that is hard to explain in terms of the input features.
Thus, the motivation for using standard decision trees.

For the sake of simplicity, throughout of this text, by explainable clustering
we mean a clustering that is obtained via decision trees.

\remove{\cite{dasgupta2020explainable} informally define the price of explainability as "the multiplicative blowup in k-means (or k-medians) cost that is inevitable if we force our final clustering to have an interpretable form." 
A broader and more formal definition of the price of explainability
$\rho({\cal P})$ for  a clustering problem ${\cal P}$, with a minimization objective function,
 may be as follows: }

\remove{ 
  Let ${\cal P}$ be a clustering problem, with a minimization objective function, that falls into the discussed framework.
The price of explainability  for ${\cal P}$ in our setting  is  defined as
$$ \max_{I } \left \{ \frac{OPT_{exp}(I)}{OPT_{unr}(I)} \right \}, $$
where $I$ runs over all instances of ${\cal P}$;
 $ OPT_{exp}(I)$ is the cost of an optimal explainable (via decision trees) clustering  
for instance $I$ and  $OPT_{unr}(I)$  is the cost of an optimal 
unrestricted clustering  for $I$.
If ${\cal P}$ has a maximization objective function then the price of explainability 
is defined as
$$ \max_{I } \left \{ \frac{OPT_{unr}(I)}{OPT_{exp}(I)} \right \}.$$
}

\subsection{Our contributions}

In Section \ref{sec:hardness}, 
we first show that the problem of building a
partition via decision trees that minimizes  the $k$-means cost function  does not admit an $(1+\epsilon)$-approximation in polynomial time, for some $\epsilon>0$, unless $P=NP$.
Then, we show that  analogous results hold for both the $k$-median and $k$-centers cost functions. 
Our results for both the $k$-means and $k$-medians are stronger than the NP-Hardness
result established recently by 
\cite{bandyapadhyay2021good} and they formally help to justify the quest for approximation algorithms and/or heuristics for these cost
functions.

In Section \ref{sec:complexity-max-spacing} we propose a polynomial time algorithm 
that produces an explainable clustering of maximum spacing. As far as we know,  this is the first
efficient method that produces optimal explainable clustering with respect to some well studied metric.





\subsection{Related work}


Our research is inspired by the recent work of \cite{dasgupta2020explainable},
where the problem of building  explainable clusterings, via standard decision trees, for both the $k$-means and the $k$-medians cost functions are studied. This paper  proposes algorithms with provable approximation bounds for building explainable clusterings. In addition, 
it investigates the price of explainability for
these cost functions, which is the unavoidable gap between the cost of the optimal explainable and
the optimal unconstrained
clustering. Among their results, they showed 
that the price of explainability for the $k$-means and $k$-median are respectively 
$O(k^2)$ and $O(k)$.

Their results were refined/improved by a series of 
 recent  papers  \cite{DBLP:conf/icml/LaberM21,DBLP:conf/icml/MakarychevS21,esfandiari2021almost,gamlath2021nearly,charikar2021near}.
Currently, the best upper bound  
for the $k$-medians is  $O(\log k \log \log k )$  \cite{DBLP:conf/icml/MakarychevS21,esfandiari2021almost} 
 while  for the $k$-means is   $O(k  \log k )$ \cite{esfandiari2021almost}.
The study of bounds that depend on the dimension $d$ was initiated in 
\cite{DBLP:conf/icml/LaberM21}, where the authors present an    $O(d \log k)$
upper bound for the $k$-medians and  an  $O(d k \log k)$ upper bound for the  $k$-means.
These bounds were improved to  $O(d \log^2 d) $ for the $k$-medians   \cite{esfandiari2021almost} and
$O( k^{1-2/d} poly (d, \log k))$ \cite{charikar2021near} for the $k$-means.

The price of explainability was also investigated for other cost functions.
In \cite{DBLP:conf/icml/LaberM21}, Laber and Murtinho considered   the $k$-centers and maximum-spacing cost functions.
In \cite{DBLP:conf/icml/MakarychevS21}, Makarychev and Shan considered the $k$-medoids problem ($k$-median with $\ell_2$ objective). Finally,   in \cite{gamlath2021nearly}, Gamlath et. al addressed
$\ell_p^p$ objectives.

The aforementioned papers, except \cite{DBLP:conf/icml/LaberM21} which also
presents experiments, are mainly theoretical. However, there are also a number of papers that
propose algorithms (without theoretical guarantees) for building explainable
clustering, among them we cite \cite{liu2000clustering,fraiman2013interpretable,bertsimas2020interpretable}.

The computational complexity of building explainable clustering via decision trees 
for  both the $k$-means and the $k$-medians problems 
is studied in \cite{bandyapadhyay2021good}.
It is shown that both problems  admit polynomial time algorithms
when either $k$ or $d$ is constant and they are	 NP-Complete
for arbitrary $k$ and $d$. In addition,
they show that an optimal explainable clustering cannot be found
in $f(k) \cdot |{\cal X}|^{o(k)}$ 
time for any computable function $f(·)$, unless Exponential Time Hypothesis (ETH) fails.

When we turn to  standard (non-explainable) clustering,  
the problems of optimizing 
the $k$-means, $k$-medians and $k$-centers cost functions are APX-Hard \cite{DBLP:journals/siamcomp/MegiddoS84,DBLP:conf/compgeom/AwasthiCKS15,DBLP:conf/soda/Cohen-AddadSL22} and  all of them 
admit polynomial time algorithms with constant approximation \cite{DBLP:journals/tcs/Gonzalez85,DBLP:journals/siamcomp/AhmadianNSW20,DBLP:journals/jcss/CharikarGTS02}.
With regards to the spacing cost function, 
the single-link algorithm, a very popular algorithm to build hierarchical clustering, 
produces a partition with maximum spacing
\cite[Chapter 4]{DBLP:books/daglib/0015106}.






\section{Hardness of $k$-means, $k$-medians and $k$-centers cost function}
\label{sec:hardness}
  
 \subsection{Background}
 
\label{sec:gappres}

 We start by recalling some basic definitions and facts that are useful for 
 studying the hardness   of optimization problems (see, e.g., \cite[chapter 29]{DBLP:books/daglib/0004338}).

Given a minimization problem $\mathbb{A}$ and a parameter $\epsilon>0$ 
we define the {\sc $\epsilon$-Gap}-$\mathbb{A}$ problem as the problem of 
deciding for an instance $I$ of $\mathbb{A}$ and a parameter $k$ 
whether: (i) $I$ admits a solution of 
value at most $k$; or (ii) every solution of $I$ have value at least $(1+\epsilon) k.$ In such a 
gap decision problem it is tacitly assumed that the instances are either of type (i)
or of type (ii).

\begin{fact} \label{fact:gap-inapprox}
If for a minimization problem $\mathbb{A}$ there exists $\epsilon > 0$ such that the 
{\sc $\epsilon$-Gap}-$\mathbb{A}$ problem is $NP$-hard, then no polynomial time
$(1+\epsilon)$-approximation algorithm exists for $\mathbb{A}$ unless $P = NP.$
\end{fact}

We will use the following definition of a gap-preserving reduction.
\begin{definition} \label{gap-reduction}
Let $\mathbb{A}, \mathbb{B}$ be minimization problems. A gap-preserving reduction from $\mathbb{A}$
to $\mathbb{B}$ is a polynomial time algorithm that, given an instance $x$ of $\mathbb{A}$ and a value $k$, produces 
an instance $y$ of $\mathbb{B}$ and a value $\kappa$ such that  there exist constants $\epsilon, \eta > 0$ for which 
\begin{enumerate}
\item if $OPT(x) \leq k$ then $OPT(y) \leq \kappa$;
\item if $OPT(x) > (1+\epsilon) k$ then $OPT(y) > (1+\eta) \kappa$; 
\end{enumerate}
\end{definition}

\begin{fact} 
Fix minimization problems  $\mathbb{A}, \mathbb{B}$. If there exists 
$\epsilon$ such that  the {\sc $\epsilon$-Gap}-$\mathbb{A}$ problem is $NP$-hard
and there exists a  gap-preserving reduction from $\mathbb{A}$
to $\mathbb{B}$ then there exists $\eta$ such that 
the {\sc $\eta$-Gap}-$\mathbb{B}$ problem is $NP$-hard
\end{fact}

 We will now specialize the above definitions for  restricted variants of 
 the problem of finding a minimum vertex cover in a graph and for our 
 clustering problems.

\begin{definition}
For every $\epsilon > 0$, the
{\sc $\epsilon$-Gap-MinVC-B-TF}
 (gap) decision problem is defined as follows:
given a triangle-free graph $G = (V, E)$, with bounded degree,   and an integer $k$, decide
whether $G$  has a vertex cover of size $k$ or all vertex covers of $G$ 
have size at least $ k (1+\epsilon)$.
\end{definition}

The {\sc $\epsilon$-Gap-MinVC-3B-TF} (gap) decision problem
has a similar definition, the only differences is that, in addition of being triangle-free, the graphs
are required to be  3-bounded, that is, all of its vertexes have degree at most 3.

The NP-Hardness of {\sc $\epsilon$-Gap-MinVC-B-TF} and 
 {\sc $\epsilon$-Gap-MinVC-3B-TF} were established in \cite{DBLP:conf/compgeom/AwasthiCKS15} and \cite{DBLP:journals/tit/CicaleseL21}, 
 respectively.

\begin{definition}
For every $\eta > 0$, the
{\sc $\eta$-Gap-Explainable-$k$means}
(gap) decision problem is defined as follows:
given a set of points ${\cal X}$, an integer $k$, and a value $\kappa$,  decide
whether there exists an explainable $k$-clustering ${\cal C} = (C_1, \dots C_k)$ of the points in ${\cal X}$  such that the
$k$-means cost of  ${\cal C}$ is  at most $\kappa$ or 
for each explainable $k$-clustering ${\cal C}$ of ${\cal X}$ it holds that 
the $k$-means
cost of  ${\cal C}$ is at least $ (1+\eta)\kappa.$
\end{definition}
The {\sc $\eta$-Gap-Explainable-kmedians}
 and {\sc $\eta$-Gap-Explainable-kcenters} decision problems
 are analogously defined.

To prove the hardness for the $k$-means 
we use a gap-preserving reduction from the {\sc $\epsilon$-Gap-MinVC-B-TF}
decision problem.
To handle both the  $k$-centers and $k$-medians,  we 
use 
 the {\sc $\epsilon$-Gap-MinVC-3B-TF} decision problem.

Our  reductions have some common  ingredients  that we explain here.
For all of them, 
given a graph $G=(V,E)$, where $V=\{1,\ldots,n\}$, we build an
instance of the clustering problem under consideration by mapping  
 every edge $e$ in $E$  onto
a point $\vvec^e=(v^e_1,\ldots,v^e_{n})$ in $\{0,1\}^n$ where  $v^e_i=1$  if vertex $i$ is incident on $e$ and $v^e_i=0$
otherwise. This is exactly the mapping proposed in \cite{DBLP:conf/compgeom/AwasthiCKS15}
to establish that the (standard) $k$-means problem is APX-Hard.
We use ${\cal X}_G:=\{v^e|e \in E\}$ to denote the input
of the resulting clustering instance.

Let $S=\{i_1,i_2,\ldots,i_k\}$ be a cover  of size $k$ for  $G$, where each $i_j$ is an integer
in $[n]$ and  $i_j < i_{j+1}$.
We define  ${\cal C}_S=(E_1,\ldots,E_k)$ as the $k$-clustering induced by $S$ on  the points in 
${\cal X}_G$,
where
 the group $E_j$ includes all 
points $\vvec$  that simultaneously satisfy: its component
$i_j$ is $1$  and  its component $i_{j'}$, for $j'<j$,
is  0.  

\begin{prop}
The  clustering ${\cal C}_S$ is explainable.
\end{prop}
\begin{proof}
${\cal C}_S$ is the clustering induced  by 
a decision tree with $k-1$  internal nodes, with exactly one internal node per level.
The internal node of level $j$ is associated with cut
$(i_j, 1/2)$. 
\end{proof}

\subsection{Hardness of $k$-means cost function}
\label{sec:complexity}
 
We prove that the problem of finding an explainable clustering
with minimum $k$-means cost function is hard to approximate.
The reduction employed here is the one used by \cite{DBLP:conf/compgeom/AwasthiCKS15} to show that 
it is hard to find an $(1+\epsilon)$-approximation
for the $k$-means cost function. 
The  extra ingredient in our  proof 
is the construction of an explainable clustering ${\cal C}_S$
from a vertex cover $S$
that was described in the previous section.

\begin{thm} \label{thm:no_approx}
The problem of building an explainable clustering, via decision trees, that minimizes the $k-$means cost function does not admit
an $(1+\epsilon)$-approximation, for some $\epsilon >0$, in
polynomial time unless $P=NP$.
\end{thm}
\begin{proof}
\label{sec:proof_thm}
Let $G$ be a triangle-free graph that satisfies one of the following 
cases: (i) $G$  has a vertex cover of size $k$ or (ii)  all vertex covers of $G$ 
have size $> k (1+\epsilon)$.

First, we consider the case where $G$ has a vertex cover $S=\{i_1,i_2,\ldots,i_k\}$ of size $k$.
We show  that, in this case, the cost of   ${\cal C}_S=(E_1,\ldots,E_k)$ is at most $|E|-k$.
 Let us consider the mean of the points in $E_j$ as the representative of $E_j$, that is, a point 
  that has $1$ at coordinate $i_j$ and 
$1/|E_j|$ in the remaining $|E_j|$ coordinates
with non-zero values.  
The squared distance of each point in $E_j$ to its representative is
given by 
\begin{equation} 
\label{eq:distance}
\left (1- \frac{1}{|E_j|}\right )^2 + (|E_j|-1) \times  \left (\frac{1}{|E_j|}\right)^2 =
  1- \frac{1}{|E_j|}  
\end{equation}

Thus, $E_j$ contributes to the
total cost with $ |E_j|-1$.
The cost of the clustering  ${\cal C}_S$ is, then, given by
$$ \sum_{j=1}^k |E_j|-1=|E|-k $$ 

Now, it remains to argue that if
the minimum vertex cover for $G$ has size at least $(1+\epsilon)k$ then 
every explainable clustering for the corresponding instance has cost at least $|E|- (1-\Omega(\epsilon))k$. This follows from 
\cite{DBLP:conf/compgeom/AwasthiCKS15}, as in this case every clustering (and, in particular, every explainable one) has cost at least $|E|- (1-\Omega(\epsilon))k$.
 
We have concluded a gap preserving reduction from 
{\sc $\epsilon$-Gap-MinVC-B-TF} to  
{\sc $\eta$-Gap-Explainable-$k$means}.
\end{proof}

\subsection{Hardness of $k$-medians cost function}
\label{sec:complexity-kmedians}
We prove that the problem of finding an explainable clustering
with minimum $k$-medians cost function is hard to approximate.
We show a gap preserving reduction from the {\sc $\epsilon$-Gap-MinVC-3B-TF}
problem to the {\sc $\eta$-Gap-Explainable-kmedians} problem.

The following well-known fact will be useful.

\begin{fact}
Let $C$ be a set of points in $ \mathbb{R}^d$ and let  $\mu(C) \in  \mathbb{R}^d$ be the point for which 
$$ \sum_{ \mathbf{x} \in C} || \mathbf{x}  - \mu(C)  ||_1 $$
is minimum. 

Then, for each $i \in [d]$, the value of coordinate $i$ of point  $\mu(C)$  is the median of the values of  the
 points in $C$
on coordinate $i$.
\end{fact}

The following lemma will be also useful.

\begin{lemma}
Let $G$ be a 3-bounded  triangle free graph and let 
let $C \subseteq {\cal X}_G$  be a group of points corresponding to  $p$ edges of $G$.
We have that: (i) if   $C$ is a star then its $k$-medians cost is $p$ and (ii) if $C$  is not a star then its 
$k$-medians cost is at least
$(4/3) p$.
\end{lemma}
\begin{proof}
From the previous fact, the representative  of $C$ that yields to the  minimum $k$-medians cost is a point in $\{0,1\}^n$, where the coordinate $i$ has value 1 
if and only if the number of edges
that touch vertex $i$ is larger than $p/2$.  
Thus, the cost of a cluster $C$ is given by
$$\sum_{i=1}^n  \min \{ p-d_C(i), d_C(i)\},$$   
where $d_C(i)$ is the number of edges that touch vertex $i$ in $C$.

If $C$ is a star centred on vertex $j$ then 
$ \min \{ p-d_C(j), d_C(j)\}=0$ and $ \min \{ p-d_C(i), d_C(i)\}=1$
for the other vertexes $i$ in the star. Thus,
$$\sum_{i=1}^n \min \{ p-d_C(i), d_C(i)\} = \sum_{i\ne j} 1 = p  $$   

If $C$ is not a star then we have some cases:

\medskip

Case 1) $d_C(i) \le p/2$ for all $i$. We have 
$$\sum_{i=1}^n  \min \{ p-d_C(i) , d_C(i)\} =\sum_{i=1}^n  d_C(i)=2p$$   

Note that the above case covers the case $p \ge 6$ since the maximum degree in $G$ is at most 3.
\medskip

Case 2)  $p= 5$ and $d_C(j) =3$ for exactly one $j$. 
 We  have  
$$\sum_{i=1}^n  \min \{ p-d_C(i), d_C(i)\} =d_C(j)-1+ \sum_{i \ne j}  d_C(i)=2p-1=9=1.8p $$ 

\medskip

Case 3)  $p= 5$ and $d_C(j) =d_C(j')=3$ for exactly two values  $j$ and $j'$. 
 We have
$$\sum_{i=1}^n  \min \{p-d_C(i), d_C(i)\} =4+ \sum_{i \notin \{j,j'\}}  d_C(i)=2p-2=8=1.6p $$ 

Note that we cannot have 3 vertexes with degree 3 and $p=5$.
\medskip

Case 4)  $p= 4$ and $d_C(j) = 3$ for some $j$.
We must have exactly one $j$ with $d_C(j) = 3$, otherwise
we would have more than $4$ edges. Thus,

$$\sum_{i=1}^n  \min \{p-d_C(i), d_C(i)\} = d_C(j)-2 + \sum_{i \ne j}  d_C(i) = 2p-2=6=1.5p $$ 

\medskip
Case 5)  $p= 3$ and $d_C(j) = 2$ for some $j$. 
 We have two possible non-isomorphic graphs. One of them consists of a path with 2 edges and an additional edge while the 
  other is a path with 3 edges. For both cases we have
$$\sum_{i=1}^n  \min \{p-d_C(i), d_C(i)\}   \ge 4 = (4/3)p $$ 
\end{proof}

\begin{thm}
The problem of building an explainable clustering, via decision trees, that minimizes the $k-$medians cost function does not admit
an $(1+\epsilon)$-approximation,  for some $\epsilon >0$, in
polynomial time unless $P=NP$.
\end{thm}
\begin{proof}
Let $G$ be a triangle-free graph with maximum degree not larger than 3 that satisfies one of the following 
cases: (i) $G$  has a vertex cover of size $k$ or (ii)  all vertex covers of $G$ 
have size at least $ k (1+\epsilon)$.

First, consider the case where $G$ has a vertex cover $S$ of size $k$.
Since the clustering ${\cal C}_{S}$ consists
of stars, it follows from the previous lemma that its cost  is $|E|$.

Now, 
 assume that all vertex  covers for $G$ have
size at least $k(1+\epsilon)$.
Let ${\cal C}$ be a clustering with $k$ groups for the corresponding $k$-medians
instance.

 Let $t$ be the number of groups in ${\cal C}$  that  are
stars and let $p$ be the total number of edges in the remaining clusters.
Since  there is no vertex cover for $G$ of
size smaller than $k(1+\epsilon)$ we must have 
$$ t + p \ge k(1+\epsilon),$$
otherwise we could obtain a cover for $G$  with size smaller than $k(1+\epsilon)$ by  using one vertex per star and one  additional
vertex for each of the $p$ edges.
Since $t \le k$ it follows that $p \ge k \epsilon$.
Moreover, we must have $k \ge|E|/3$ because the degree of every vertex in  $G$
is at most $3$. 
Thus, from the previous lemma,  the cost of clustering ${\cal C}$ is at least
$$ \frac{4p}{3}  + (|E|-p)=|E|+\frac{p}{3} \ge |E|+ \frac{k \epsilon}{ 3} \ge |E|+ \frac{|E| \epsilon}{ 9}=
|E|\left (1+ \frac{\epsilon}{ 9} \right). $$

We have concluded a gap preserving reduction from 
{\sc $\epsilon$-Gap-MinVC-3B-TF} to  
{\sc $\eta$-Gap-Explainable-$k$medians}.

\end{proof}

\subsection{Hardness of $k$-centers cost function}
\label{sec:complexity-kcenter}
In this section we discuss the computational complexity of minimizing the $k$-centers cost
function.
We show a gap preserving reduction from the {\sc $\epsilon$-Gap-MinVC-3B-TF} problem
to the {\sc $\eta$-Gap-Explainable-kcenters}
 problem.

\begin{thm} \label{thm:no_approx-kcenter}
The problem of building an explainable clustering, via decision trees, that minimizes the $k-$centers cost function does not admit
an $(1+\epsilon)$-approximation, for some $\epsilon >0$, in
polynomial time unless $P=NP$.
\end{thm}
\begin{proof}
Let $G=(V,E)$ be a triangle-free graph with maximum degree not larger than 3
that satisfies one of the following 
cases: (i) $G$  has a vertex cover of size $k$ or (ii)  all vertex covers of $G$ 
have size at least $ k(1+\epsilon)$.

First, consider the case where $G$ has a vertex cover $S$ of size $k$.
In this case, the clustering ${\cal C}_S=(E_1,\ldots,E_k)$ consists
of stars with at most 3 edges.
For the representative of $E_j$, as in the  proof of Theorem \ref{thm:no_approx}, we use the mean of the points that lie
in $E_j$.

Thus, the distance of each point in $E_i$ to its representative is
the square root of the rightmost term of  (\ref{eq:distance}), which is at most $\sqrt{1/3}$
since $G$ has maximum degree 3.

Now, we assume that $G$ does not have a vertex cover with $k$ vertex.
Let ${\cal C}$ be a clustering with $k$ groups for the edges of $E$.
One of the groups, say $A$, does not have a vertex that touches all the edges in $A$. 
 Pick the vertex, say $v$, that touches the largest number of edges in $A$.
Consider an edge $e=yz$ in $A$ that does not touch $v$. 
We show that there is another 
edge in $A$, say $e'$, that does not have intersection with $e$.
In fact, pick an edge $f=vw$. If $f$ does not intersect
$e$ ($w$ is not an endpoint of $e$) we set $e'=f$. Otherwise, we assume w.l.o.g. that  $f$ intersects $e$ at point $y$, that is, $w=y$. We
know that $vz$ is not an edge for otherwise we would have a triangle $vwz$ in $G$.
Since $v$ is the vertex that touches the largest number of edges in $A$ then
$v$ must touch an edge $f'=vz'$, with $z' \ne y$ and $z' \ne z$. 
We set $e'=f'$.

We can argue that the distance of the representative $\mu(A)$ of $A$ to either $e$ or $e'$ is at least
$1$. For that, we  consider the values of  $\mu(A)$
at the components of the vertexes that define the edges $e'$ and $e$. 
Let $\mu_1,\mu_2,\mu_3$ and $\mu_4$ be these values. 
We have that 
$$||e-\mu(A)||^2 +||e'-\mu(A)||^2  \ge  \sum_{i=1}^4 (1-\mu_i)^2 + \mu_i^2 = $$
$$ 4 -2(\mu_1+\mu_2+\mu_4+\mu_4)+2(\mu_1^2 +\mu_2^2+\mu^2_3+\mu^2_4) \ge 2 $$
Thus, either $e$ or $e'$ is at distance at least 1 from the representative of $A$
\end{proof}

\section{A polynomial time algorithm for the maximum-spacing cost function}
\label{sec:complexity-max-spacing}
We describe {\tt MaxSpacing}, a simple greedy algorithm that  finds
an explainable partition of maximum spacing in polynomial time.

To simplify its  description we introduce some notation.  For a set of leaves $L$ in a decision tree, we use $sp(L)$  to 
refer to the spacing of the partition of the points in ${\cal X}$ induced by the leaves in  $L$.
Given a set of leaves $L$, a leaf $\ell \in L$ and an axis-aligned cut $\gamma=(i, \theta)$, we use $L_{\gamma,\ell}$
to denote the set of leaves obtained when $\gamma$ is applied to split the points that reach $\ell$.
More precisely, $L_{\gamma,\ell}$ is obtained from $L$ by removing $\ell$   and adding   the two leaves
that are created by using $\gamma$ to split the points that reach $\ell$.

A  pseudo-code for {\tt MaxSpacing} is presented in Algorithm \ref{alg:maximum-spacing}.
The algorithm adopts a natural greedy strategy that at each step chooses the cut that yields to the partition
of maximum spacing.
We note that it runs in polynomial time because in Step 1 we just need to test at most $(|{\cal X}|-1)d$ axis-aligned cuts: for each $\ell \in L$ and each dimension $i \in [d]$ we sort the $|\ell|$ points
that reach $\ell$
according to their coordinate $i$ and consider the cuts
$(i,\theta_j)$, for $j=1,\ldots,|\ell|$, where  $\theta_j$ is the midpoint between
the values of the $i$-th coordinate of the $j$th  and $(j+1)$th points in the sorted list.

\begin{algorithm}[H]
  \caption{{\tt MaxSpacing}(${\cal X}$: set of points; $k$: integer)}

   \begin{algorithmic}[]
  	
  	\small

\medskip

      \STATE Initialize a decision tree with only one leaf $\ell$ and associate it with ${\cal X}$  
      
      \STATE $L \gets \	\{\ell\}$

		\STATE Repeat $k-1$ times:

\begin{enumerate}
\item Find a cut $\gamma$ and a leaf $\ell \in L$ that simultaneously satisfy: 
\begin{itemize}
\item[(i)]
 $\gamma$ splits the points that reach $\ell$ into 2 non-empty groups
\item[(ii)] $sp(L_{\gamma,\ell}) \ge sp(L_{\gamma',\ell'}) $
for every $\ell' \in L$ and every axis-aligned cut $\gamma'$ that splits the points that
reach $\ell'$ into two non-empty sets.
\end{itemize}
\item Split leaf $\ell$ using cut $\gamma$

\item $ L \leftarrow L_{\gamma,\ell}$
 \end{enumerate}  

  \end{algorithmic}
  \label{alg:maximum-spacing}
\end{algorithm}

In what follows, we show that {\tt MaxSpacing} produces an explainable partition with maximum
possible (optimal) spacing.
The following simple fact  will be useful.

\remove{
\begin{fact} 
\label{fact:simple}
Let $A,B,C$ and $D$ be a set of points in ${\cal X}$. If  $A \subseteq C$ and $B \subseteq D$ then
$sp(A,B) \ge sp(C,D)$.
\end{fact}  
\begin{proof}
Let $a \in A$ and $b \in B$ be such that $dist(a,b)=sp(A,B)$.
Since $a \in C$ and $b \in D$ it follows that $sp(C,D) \le dist(a,b)=sp(A,B)$ 
\end{proof}
}

\begin{fact} 
\label{fact:simple2}
Let $d^*_i$ be the spacing of an optimal explainable partition with $i+1$ groups.
Then, $  d^*_i \ge d^*_{i+1}$, for $i=1,\ldots,k-1$.
\end{fact}  
\begin{proof}
Let $D^*$ be a decision tree that induces a partition with $i+2$ groups
that has spacing $d^*_{i+1}$. Let $D$ be a decision tree obtained by removing 
two leaves that are siblings in $D^*$ and turning their parent into a leaf. 
Let $x$ and $y$ be two closest points among those that reach different
leaves in $D$. Since these points also reach distinct leaves in 
$D^*$ we have that the spacing of  the leaves in $D$ is not smaller
than that of the leaves in $D^*$. Thus,
$d^*_i \ge sp( \mbox{Leaves of } D) \ge d^*_{i+1}$.
 \end{proof}

\begin{thm}
For every $1 \le i \le k-1$,
the
partition induced by the leaves of   {\tt MaxSpacing} algorithm
by the end of  iteration $i$ has the maximum spacing, among the
explainable partitions with $i+1$ groups for ${\cal X}$.
\label{thm:spacing}
\end{thm}
\begin{proof}

Let ${\cal C}^*_i$ be  an optimal explainable partition with $i+1$ groups and let $d^*_i$ be its  spacing.
Moreover, let $L_{i}$, with $i<k-1$, be the set of leaves by the end of iteration $i$ of {\tt MaxSpacing} algorithm.
By the greedy choice $sp(L_1)= d^*_1$.
We assume by induction that the spacing of  $L_i$ is  $d^*_{i}$ and 
show that the spacing of  $L_{i+1}$  is  $d^*_{i+1}$.


For a node $\nu$ in a decision tree, let $P(\nu)$ be the set of points that reach $\nu$.
Let $\nu^*$ be a node in the decision tree $D^*$ for ${\cal C}_{i+1}^*$ that satisfies the following:
(i) for each $\ell \in L_i$ either $P(\ell) \subseteq P(\nu^*)$ or
$P(\ell) \cap P(\nu^*) =\emptyset$ and (ii) some child of $\nu^*$ does not satisfy (i).
We will use the cut associated with  $\nu^*$  in $D^*$ to argue that we can properly split $L_i$.

To prove the existence of a  node $\nu^*$ with such properties, it suffices to show  that the root $r^*$ of $D^*$ satisfies $(i)$ and
some leaf $\ell^*$ from  $D^*$  does not satisfy (i) since, in this case,  we can  set $\nu^*$ as the last node in the path from $r^*$  to $\ell^*$ that satisfies (i).
Clearly, $r^*$ satisfies (i). 
It remains to argue that some leaf $\ell^*  \in D^*$ does not satisfy (i). 
Since the number of leaves in $L_i$ is smaller than the number of leaves in $D^*$, 
by the pigeonhole principle, there are two leaves, say $\ell^*_1$ and $\ell^*_2$, in $D^*$ that
contain points from the same leaf $\ell$ in $L_i$. Thus,  neither
$P(\ell) \cap P(\ell^*_1) \ne \emptyset$ nor  $P(\ell) \subseteq P(\ell^*_1)$.
We set $\ell^*$ to $\ell^*_1$ and $\nu^*$   as the last node in the path from $r^*$  to $\ell^*_1$ that satisfies (i).


Let $\nu^*_{ch}$ be a child of $\nu^*$ in $D^*$ that does not satisfy (i).
Moreover, let  $\gamma$ be the cut associated with $\nu^*$
and let $\ell$ be a leaf in $L_i$ such that 
$P(\ell) \subseteq P(\nu^*)$, $P(\ell) \cap P(\nu^*_{ch}) \ne \emptyset$
and $P(\ell)\not\subset P(\nu^*_{ch}) $.
We show that the spacing of the set of leaves $L'_i$ obtained from $L_i$
by applying cut  $\gamma$ to $\ell$ is at least
$d^*_{i+1}$. 
Let $\ell_1$ and $\ell_2$ be the two new leaves that are created  by applying
$\gamma$ to $\ell$ and 
let $\mathbf{x}$ and $\mathbf{y}$
be the two closest points (according to $dist$) among those that reach
different leaves in  $L'_i$. 
If $\mathbf{x}$ reaches $\ell_1$ (resp. $\ell_2$)  and $\mathbf{y}$ reaches
$\ell_2$ (resp. $\ell_1$) then $$sp(L'_i)=dist(\mathbf{x},\mathbf{y}) \ge sp({\cal C}_{i+1}^*) = d^*_{i+1}$$ because,
due to the application of  cut $\gamma$ on $\nu^*$,
$\mathbf{x}$ and $\mathbf{y}$ lies in different groups in ${\cal C}_{i+1}^*$.
If some of them, say $\mathbf{x}$, does not 
reach  $\ell$ then $\mathbf{x}$ and 
$\mathbf{y}$ reach different leaves in $L_i$ and, thus,
$$sp(L'_i) = dist(\mathbf{x},\mathbf{y}) \ge sp(L_i) = d^*_i \ge d^*_{i+1}, $$
where the last inequality follows from Fact \ref{fact:simple2}.

We have shown that there exists a leaf $\ell$ in $L_i$ and a  cut $\gamma$ 
such that the application of $\gamma$ to $\ell$ yields to a partition of spacing $d^*_{i+1}$.
Thus, due to  the greedy choice, {\tt MaxSpacing} obtains a partition of spacing $d^*_{i+1}$
by the end of iteration $i+1$.
\end{proof}

\section{Conclusions}
We have showed that the problems of finding explainable 
clustering (via decision trees) that optimize  the classical $k$-means,
$k$-medians and $k$-centers cost functions do not admit polynomial time $(1+\epsilon)$-approximations.
These results help to formally justify the quest for heuristics and/or approximation algorithms.

The algorithms recently proposed in the literature 
for building explainable clustering  compare their costs with the
costs of optimal unrestricted clustering
\cite{dasgupta2020explainable,DBLP:conf/icml/LaberM21,DBLP:conf/icml/MakarychevS21,charikar2021near,esfandiari2021almost,gamlath2021nearly}.
  A major open question in this line of research is
whether better bounds can be obtained when the comparison is made against the 
optimal explainable clustering.

For the spacing cost function we provided a simple polynomial time algorithm
that computes the explainable partition with maximum spacing.
An interesting note is that we have not used the fact that the cuts are axis-aligned in the proof of
Theorem \ref{thm:spacing} and,  thus, our result holds for any family of cuts.


\bibliography{biblio}

\end{document}